\documentclass[english,10pt]{article}
\usepackage[T1]{fontenc}
\usepackage[latin9]{inputenc}
\usepackage{geometry}
\usepackage{babel}
\usepackage{verbatim}
\usepackage{float}
\usepackage{bm}
\usepackage{amsmath}
\usepackage{amssymb}
\usepackage{graphicx}
\usepackage{breakurl}
\usepackage{amsthm}
\usepackage{amsfonts}
\usepackage{fullpage}
\usepackage{times}
\usepackage{array}
\usepackage{mathtools}

\RequirePackage{natbib}
\RequirePackage[colorlinks,citecolor=blue,urlcolor=blue]{hyperref}
\makeatletter

\floatstyle{ruled}

\usepackage{amsthm}\usepackage{dsfont}\usepackage{array}\usepackage{mathrsfs}\usepackage{comment}\onecolumn

\usepackage{color}\usepackage{babel}

\usepackage{babel}
\usepackage{arydshln}
\usepackage{enumitem}
\usepackage{algorithm}
\usepackage{algpseudocode}

\makeatother

\begin{document}
\theoremstyle{plain} \newtheorem{lemma}{\textbf{Lemma}} \newtheorem{prop}{\textbf{Proposition}}\newtheorem{theorem}{\textbf{Theorem}}\setcounter{theorem}{0}
\newtheorem{corollary}{\textbf{Corollary}} \newtheorem{assumption}{\textbf{Assumption}}
\newtheorem{example}{\textbf{Example}} \newtheorem{definition}{\textbf{Definition}}
\newtheorem{fact}{\textbf{Fact}} \theoremstyle{definition}

\theoremstyle{remark}\newtheorem{remark}{\textbf{Remark}}

\title{Deep Primal-Dual Reinforcement Learning: Accelerating Actor-Critic using Bellman Duality} 

\author
{
	Woon Sang Cho\thanks{Department of Operations Research and Financial Engineering, Princeton University, Princeton 08544, USA; Email:
		\texttt{woonsang@princeton.edu}.}
	\qquad Mengdi Wang\thanks{Department of Operations Research and Financial Engineering, Princeton University, Princeton 08544, USA; Email:
		\texttt{mengdiw@princeton.edu}.}
}

\date{December 3, 2017}

\maketitle
\begin{abstract}
We develop a parameterized Primal-Dual $\pi$ Learning method based on deep neural networks for Markov decision process with large state space and off-policy reinforcement learning. In contrast to the popular Q-learning and actor-critic methods that are based on successive approximations to the nonlinear Bellman equation, our method makes primal-dual updates to the policy and value functions utilizing the fundamental linear Bellman duality.  Naive parametrization of the primal-dual $\pi$ learning method using deep neural networks would encounter two major challenges: (1) each update requires computing a probability distribution over the state space and is intractable; (2) the iterates are unstable since the parameterized Lagrangian function is no longer linear. We address these challenges by proposing a relaxed Lagrangian formulation with a regularization penalty using the advantage function. We show that the dual policy update step in our method is equivalent to the policy gradient update in the actor-critic method in some special case, while the value updates differ substantially. The main advantage of the primal-dual $\pi$ learning method lies in that the value and policy updates are closely coupled together using the Bellman duality and therefore more informative. Experiments on a simple cart-pole problem show that the algorithm significantly outperforms the one-step temporal-difference actor-critic method, which is the most relevant benchmark method to compare with. We believe that the primal-dual updates to the value and policy functions would expedite the learning process. The proposed methods might open a door to more efficient algorithms and sharper theoretical analysis.
\end{abstract}

\section{Introduction}

Linear programming has been known for decades as a classical approach for Markov Decision Process (MDP).  While the value and policy iterations for MDP lead to the popular Q-learning and actor-critic methods for reinforcement learning, it has remained open whether linear programming approaches  apply to reinforcement learning at all. We consider a discounted MDP, which can be described by a tuple $\mathcal{M} = (\mathbb{S}, \mathbb{A}, \mathbb{P}, \bf{r}, \gamma)$. The optimal value function $V^* \in\Re^{\mathbb{S}}$ satisfies the nonlinear Bellman equation: 
$$
V^{*}(s) 
= \max_{\pi} \mathbb{E}_{\pi} \Bigg[ \sum_{t=1}^{\infty}\gamma^{t-1}r_{t} \vert s_{0} = s \Bigg]
= \max_{\pi} \mathbb{E}_{a \sim \pi(\cdot \vert s)} \Bigg[ \sum_{s' \in \mathbb{S}} p_{s,a \rightarrow s'} r + 
	\gamma \sum_{s' \in \mathbb{S}} p_{s,a \rightarrow s'} V^{\pi}(s')
	 \Bigg] 
$$ for all states $s\in\mathbb{S}$. This system of $|\mathbb{S}| \times |\mathbb{S}||\mathbb{A}|$ equations can be solved through a linear program (see \cite{Puterman:1994:MDP:528623} and \cite{doi:10.1287/opre.51.6.850.24925}):
\begin{equation*}
\begin{array}{ll@{}rr}
\text{minimize}  & \displaystyle (1-\gamma) \mathbb{E}_{s \sim q(s)} \big[V(s) \big] &\\
\text{subject to}& \displaystyle V(s) \geqslant  r(s,a) + \gamma \mathbb{E}_{s' \vert s,a}\big[V(s')\big] ,  &&\forall (s,a) \in \mathbb{S} \times \mathbb{A}                                           
\end{array}
\end{equation*}
\\
and the Lagrangian formulation of
\begin{equation} \label{eq:1}
\min_{V} \max_{\mu \geqslant 0} \mathbb{L}(V,\mu) \coloneqq (1-\gamma) \mathbb{E}_{s \sim q(s)} \big[ V(s)\big] + \sum_{(s,a) \in \mathbb{S} \times \mathbb{A}} \mu (s,a) \cdot \bigg[ r(s,a) + \gamma \mathbb{E}_{s' \vert s,a}\big[V(s')\big] - V(s)  \bigg]
\end{equation}
where $q$ is some initial distribution over the state space $\mathbb{S}$.

\subsection{Motivation}
We take motivation from a novel stochastic primal-dual algorithm that can find $\epsilon$-optimal policy in nearly-linear run time for the worst case, developed in \cite{DBLP:journals/corr/Wang17i}. However, this algorithm is developed for the tabular-setting formulation (\ref{eq:1}) and does not directly apply to problems with large or even infinite state and action spaces.  The natural question that follows is: 
$$\textit{How to make the primal-dual method work with function approximation and parametrization?}$$
\noindent
Thus we consider the following optimization problem, where primal and dual variables are parametrized by $\theta_{v}$ and $\theta_{\mu}$, respectively.
\begin{equation} \label{eq:2}
\min_{\theta_{v}} \max_{\substack{\theta_{\mu} \\ \text{s.t.} \mu_{\theta_{\mu}} \geqslant 0}} \mathbb{L}(\theta_{v},\theta_{\mu}) \coloneqq (1-\gamma) \mathbb{E}_{s \sim q(s)} \big[ V_{\theta_{v}}(s)\big] + \sum_{(s,a) \in \mathbb{S} \times \mathbb{A}} \mu_{\theta_{\mu}} (s,a) \cdot \bigg[ r(s,a) + \gamma \mathbb{E}_{s' \vert s,a}\big[V_{\theta_{v}}(s')\big] - V_{\theta_{v}}(s)  \bigg]
\end{equation}
Such a parametrization loses the bilinear structure of the Lagrangian function in the tabular case. Our experiments show that solving this problem directly using stochastic gradient descent results in unstable iterates. 
\noindent
\\\\
In what follows, we denote the advantage function of a fixed policy $\pi$ as ${A}^{\pi}(s,a) = Q^{\pi}(s,a) - V^{\pi}(s) = r(s,a) + \gamma \mathbb{E}_{s' \vert s,a}\big[V^{\pi}(s')\big] - V^{\pi}(s)$. We denote the parametrized advantage function as $A_{\theta_{v}}(s,a) =  r(s,a) + \gamma \mathbb{E}_{s' \vert s,a}\big[V_{\theta_{v}}(s')\big] - V_{\theta_{v}}(s)$. We further denote the true one-step temporal difference (TD) error as $\delta^{\pi}(s,a)=r(s,a) + \gamma V^{\pi}(s') - V^{\pi}(s)$, which is an unbiased estimate of $A^{\pi}(s,a)$, and the approximated TD error as $\delta_{\theta_{v}}(s,a)=r(s,a) + \gamma V_{\theta_{v}}(s') - V_{\theta_{v}}(s)$, which is a biased estimate of $A_{\theta_{v}}(s,a)$.

\subsection{Problem Formulation}
Motivated by the temporal difference learning mechanism and Karush-Kuhn-Tucker (KKT) conditions, in particular, complementary slackness conditions, we regularize the parametrized Lagrangian by adding a penalty term $A_{\theta_{v}}(s,a)^{2}$ and consider the revised minimax formulation.
\begin{equation} \label{eq:3}
\min_{\theta_{v}} \max_{\substack{\theta_{\mu} \\ \text{s.t.} \mu_{\theta_{\mu}} \geqslant 0}} \widetilde{\mathbb{L}}(\theta_{v}, \theta_{\mu})\coloneqq (1-\gamma) \mathbb{E}_{s \sim q(s)} \big[ V_{\theta_{v}}(s)\big] + \sum_{(s,a) \in \mathbb{S} \times \mathbb{A}} \mu_{\theta_{\mu}} (s,a) \cdot A_{\theta_{v}}(s,a) + c \cdot A_{\theta_{v}}(s,a)^{2}
\end{equation}
\noindent
Even when the problem becomes nonlinear and non-convex after parameterizing the value function with neural networks, we empirically show that it suffices to search around the neighborhood of the optimal parameters $\theta^{*}_{v}$ that meet the complementary slackness conditions: $\mu_{\theta_{\mu}}(s,a) \cdot \big( r(s,a)+\gamma V_{\theta^{*}_{v}}(s') - V_{\theta^{*}_{v}}(s) \big) = 0$, $\forall s,s' \in \mathbb{S}, \forall a \in \mathbb{A}$. 
\\\\
Using Theorem 6.9.1 in \cite{Puterman:1994:MDP:528623}, we present an alternative formulation. The theorem shows a critical relationship between a randomized stationary policy $\tilde{\pi}$ and a feasible dual solution $\tilde{\mu}(s,a)$ in that $\tilde{\pi}(a \vert s) = \frac{\tilde{\mu}(s,a)}{\sum_{a \in \mathbb{A}} \tilde{\mu}(s,a)}$. 
\\\\
If we introduce an auxiliary variable $\alpha$ as a distribution over the state space $\mathbb{S}$, then we obtain $\tilde{\pi}(a \vert s) = \frac{\tilde{\mu}(s,a)}{\tilde{\alpha}(s)} \Longleftrightarrow \tilde{\mu}(s,a) = \tilde{\alpha}(s) \tilde{\pi}(a \vert s)$. Since both $\alpha$ and $\pi$ are probability measures, the non-negativity constraint of $\mu$ trivially holds. This yields an alternative formulation of equation (\ref{eq:3}) as follows.
\begin{equation} \label{eq:4}
\min_{\theta_{v}} \max\limits_{\substack{\alpha \in \mathbb{P}(S) \\ \pi \in \mathbb{P}(A)}} \widetilde{\mathbb{L}}(\theta_{v}, \theta_{\pi})\coloneqq (1-\gamma) \mathbb{E}_{s \sim q(s)} \big[ V_{\theta_{v}}(s)\big] + \sum_{(s,a) \in \mathbb{S} \times \mathbb{A}} \alpha(s) \pi_{\theta_{\pi}}(a \vert s) \cdot A_{\theta_{v}}(s,a) + c \cdot A_{\theta_{v}}(s,a)^{2}
\end{equation}

\subsection{Methodology}
A related work is the stochastic dual ascent for solving the Bellman Lagrangian problem (see \cite{anonymous2018boosting}). Yet, this method needs to solve a sequence of partial Lagrangian maximization problems to update the value function for each policy, which could be time consuming. In addition, the dual ascent method is not guaranteed to converge due to the lack of curvature of the Bellman Lagrangian function. It also requires computing an analytic solution for an auxiliary variable, which might be intractable for large state space.  In contrast, the proposed primal-dual method is much simpler to implement and doesn't involve solving subproblems. The primal-dual method enjoys strong convergence guarantees in the tabular setting.
\\\\
Thus we naturally extend the generic primal-dual framework to the regularized Lagrangian form with differentiable primal and dual function approximators and show that our algorithm outperforms the benchmark: one-step parametrized temporal-difference actor-critic algorithm. We selected this benchmark algorithm since both algorithms use one state-transition sample to make parameter updates, without using any mini-batch or experience replay.

\section{Main Results}
\subsection{Algorithm}
We propose a generic framework for Deep Primal-Dual Policy Learning method for solving problem (\ref{eq:3}). It makes update to the policy and value networks at every sample transition. It can be extended to process mini-batches of past samples.  
\begin{algorithm}
\caption{Deep Primal-Dual Policy Learning}
\label{pseudoPSO}
\begin{algorithmic}[1]
\State \textbf{Input: } $\mathcal{M}=(\mathbb{S}, \mathbb{A}, \mathbb{P}, \bf{r}, \gamma), \eta_{v}, \eta_{\pi}, \text{T}$
\State Initialize deep value and policy network weights $\theta_{v}$ and $\theta_{\pi}$, respectively.

\For{episode $e  = 1,\dots $}
\State Initialize state $s_{0}$, and $s' \gets s_{0}$.
\For{$t = 1, \dots, \text{T}$}
	\State $s \gets s'$ , sample $a \sim \pi_{\theta_{\pi}}(\cdot \vert s)$, and observe reward $r(s,a)$ and the next state $s'$.
    \State $\delta \gets r(s,a) + \gamma V_{\theta_{v}}(s') - V_{\theta_{v}}(s)$
    \State \bf{Primal-update}: $\theta_{v} \gets \theta_{v} - \eta_{v} \widehat{\nabla}_{\theta_{v}} \widetilde{\mathbb{L}}(\theta_{v}, \theta_{\pi})$ \Comment{$\widehat{\nabla}_{\theta_{v}}$ denotes sampled gradient}
    \State \bf{Dual-update}: $\theta_{\pi} \gets \theta_{\pi} + \eta_{\pi} \nabla_{\theta_{\pi}} \log \pi_{\theta_{\pi}}(a \vert s) \cdot \delta$ \Comment{Theorem 1 in Section 2.3}
    
\EndFor
\EndFor
\end{algorithmic}
\end{algorithm}

\subsection{Empirical result}
\begin{figure}[h]
\hspace{-1.5cm}   
\includegraphics[scale=0.53]{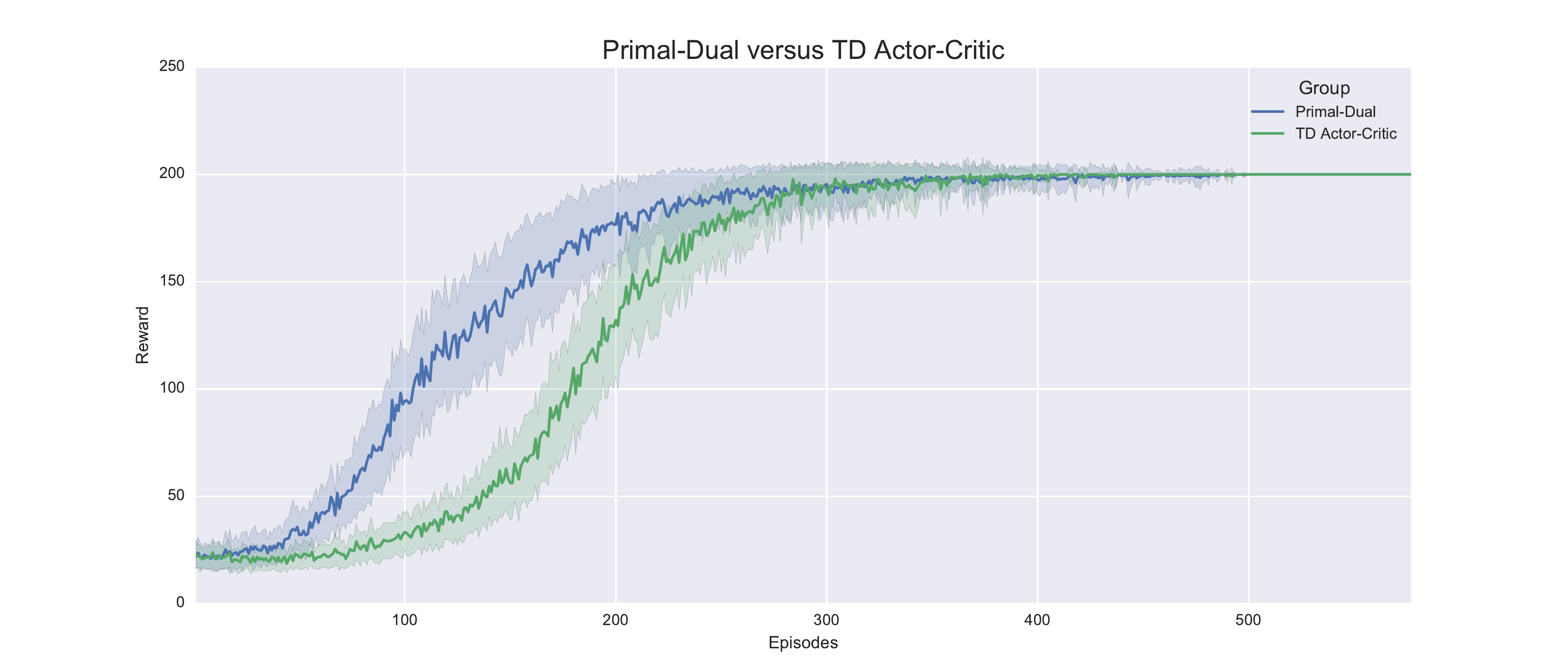}
\caption{Primal-Dual algorithm outperforms one-step TD Actor-Critic. This suggests that value updates should utilize, rather than neglect, the policy $\pi$. In other words, the critic should be informed about the actor's behavior preferences and make discreet updates, unlike previous works (e.g. \cite{DBLP:journals/corr/MnihBMGLHSK16}) in which critic learns blindly from the actor's tendencies. }
\label{fig:result}
\end{figure}
\noindent
We used the OpenAI Gym testbed (\cite{1606.01540}) to verify the algorithm on a simple CartPole-v0 problem. This problem is considered "solved" if the algorithm achieves an average cumulative reward of 195 over 100-200 consecutive episodes. The maximal achievable score is 200 per episode. For comparison, we  also employed the one-step semi-gradient TD method where the primal update uses the  targets $r(s,a) + \gamma V_{\theta_{v}}(s')$. We used a softmax activation for policy $\pi_{\theta_{\pi}}$ and $\alpha$ is assumed to be a probability distribution over the state space $\mathbb{S}$ since states are sampled according to the stationary distribution \textit{implied by $\pi_{\theta_{\pi}}$}. Therefore, it suffices to only update $\pi_{\theta_{\pi}}$ in the dual iteration step.
\\\\
For each $V$ and $\pi$ function approximators, we used neural networks of two hidden layers with 512 tanh units, learning rates $\eta_{v}=0.001$, $\eta_{\pi}=0.00001$, $\gamma = 0.99$, and $c=1$. With these hyper-parameters, the on-line algorithm "solved" the problem in approximately 200 $\sim$ 400 episodes, with variance due to initializations. 
\\\\
Furthermore, we compared the algorithm against one-step TD actor-critic. For comparison \footnote{However, we note that such a setting does not result in a completely fair comparison since the magnitude of gradients are different in the two algorithms, and we did not adjust learning rates to account for this difference. In our future work, we plan to search over a range of learning rates and compare best results from the two algorithms.}, in addition to setting the penalty coefficient $c=1$, we used the same hyper-parameters in both algorithms, and ran 100 trials for each algorithm. The shaded area shows the variance of the cumulative reward curves. For visual clarity, we scaled the standard deviation by half, therefore the figure shows $\text{mean} \pm 0.5 \cdot \sigma_{\text{std}}$. Despite slightly larger variance in cumulative reward curves from our algorithm, we observed a significant reduction in sample complexity. 
\\\\
This suggests that the the policy $\pi$ \textit{informs} which state values are more relevant for updating. Thus the value update step conditions on the policy $\pi$, whereas in previous works (e.g. \cite{DBLP:journals/corr/MnihBMGLHSK16}) estimate the value function \textit{independently of the actor's preferences}.


\subsection{Theoretical result}
\begin{theorem}
Suppose that $s$ is drawn from the stationary distribution of the current policy $\pi$.
The policy gradient update in equation \eqref{eq:4} is identical to the policy gradient update in the actor-critic method (\cite{sutton2000policy}).
\end{theorem}
\begin{proof}
Define the performance measure $\mathcal{J}(\pi) = \mathbb{E}_{s_{0} \sim q(s), \pi} \big[\sum_{t=1}^{\infty}\gamma^{t-1}r_{t} \big]$. From the policy gradient theorem (see \cite{sutton2000policy}), the gradient of $\mathcal{J}$ is 
\begin{equation} \label{eq:5}
\mathbb{E}_{s \sim \rho^{\pi}, a \sim \pi(\cdot \vert s)} \big[ Q^{\pi}(s,a) \nabla_{\theta_{\pi}} \log \pi_{\theta_{\pi}}(a\vert s) \big] 
\end{equation}
\noindent
Now we consider the dual iteration in equation (\ref{eq:1}) (see \cite{DBLP:journals/corr/Wang17i}). Without loss of generality, let $\gamma = 1$, and fix a particular value function as $\tilde{V}$. Thus $\tilde{A}^{\pi}(s,a) = \tilde{Q}^{\pi}(s,a) - \tilde{V}^{\pi}(s) = r(s,a) + \mathbb{E}_{s' \vert s,a}\big[\tilde{V}^{\pi}(s')\big] - \tilde{V}^{\pi}(s)$. The dual optimization becomes
\begin{equation} \label{eq:6}
\max\limits_{\mu \geqslant 0} \sum\limits_{(s,a) \in \mathbb{S} \times \mathbb{A}} \mu (s,a) \cdot \tilde{A}^{\pi}(s,a) 
\Longleftrightarrow  \max\limits_{\substack{\alpha \in \mathbb{P}(S) \\ \pi \in \mathbb{P}(A)}} \sum\limits_{(s,a) \in \mathbb{S} \times \mathbb{A}} \alpha(s) \pi(a \vert s) \cdot \tilde{A}^{\pi}(s,a) 
\Longleftrightarrow  \max\limits_{\substack{\alpha \in \mathbb{P}(S) \\ \pi \in \mathbb{P}(A)}} \mathbb{E}_{\alpha, \pi}\big[ \tilde{A}^{\pi}(s,a) \big]
\end{equation}
\noindent
Next note the gradient of the expression in (\ref{eq:6}) with respect to $\theta_{\pi}$, where without loss of generality, we assume $\pi$ is parametrized by $\theta_{\pi}$.
\begin{eqnarray} 
\nabla_{\theta_{\pi}} \sum_{(s,a) \in \mathbb{S} \times \mathbb{A}} \alpha(s) \pi_{\theta_{\pi}}(a \vert s) \cdot \tilde{A}^{\pi}(s,a) 
&=& \sum_{(s,a) \in \mathbb{S} \times \mathbb{A}} \alpha(s) \nabla_{\theta_{\pi}}  \pi_{\theta_{\pi}}(a \vert s) \cdot \tilde{Q}^{\pi}(s,a)  \nonumber \\
&=& \sum_{(s,a) \in \mathbb{S} \times \mathbb{A}} \alpha(s) \pi(a \vert s) \nabla_{\theta_{\pi}} \log \pi_{\theta_{\pi}}(a \vert s) \cdot \tilde{Q}^{\pi}(s,a) \nonumber \\
&=& \mathbb{E}_{\alpha, \pi} \big[\tilde{Q}^{\pi}(s,a) \cdot \nabla_{\theta_{\pi}} \log \pi_{\theta_{\pi}}(a \vert s) \big] 
\label{eq:7}
\end{eqnarray}
\noindent
Further suppose that $\alpha = \rho^{\pi}$, the stationary distribution under $\pi$, then it is identical to the policy gradient in (\ref{eq:5}).
\end{proof}

\subsection{Algorithmic implications}
A naive application of primal-dual framework requires subsequent updates to $\alpha, \pi$, and $V$. However when the state size is large, it is challenging to estimate $\alpha$, where $\sum_{s} \alpha(s) = 1$. We can avoid a direct estimation of $\alpha$ by Theorem 1, and update $\pi$ by taking on-line samples of the gradient in (\ref{eq:7}). Consequently, the policy update becomes identical to that in an actor-critic algorithm. 


\section{Conclusion}
In this work, we developed a deep primal-dual reinforcement learning algorithm, and showed that it significantly outperforms the relevant benchmark algorithm. We showed that the policy gradient update in the dual-update for the minimax formulation (\ref{eq:1}) is identical to a policy gradient update for solving an MDP. This implies we can avoid the challenging computation of $\alpha$ from the naive application of primal-dual framework. Furthermore, this work suggests that future research on critic updates in actor-critic algorithms should be informed about the actor's behavior "tendencies", rather than only consider actions already taken. 
\nocite{*}
{
\bibliographystyle{ims}
\bibliography{bibfile}
}

\end{document}